\documentclass{article}

\usepackage[final,nonatbib]{nips_2017}

\usepackage[utf8]{inputenc} 
\usepackage[T1]{fontenc}    
\usepackage{hyperref}       
\usepackage{url}            
\usepackage{booktabs}       
\usepackage{amsfonts}       
\usepackage{nicefrac}       
\usepackage{microtype}      
\usepackage{amsmath}
\usepackage{amsthm}
\usepackage{amssymb}
\usepackage{epsfig}
\usepackage{graphicx}
\usepackage{multirow}
\usepackage{color}
\usepackage{subfigure}
\usepackage{epstopdf}
\usepackage{wrapfig}

\newtheorem{theorem}{Theorem}
\newtheorem{definition}{Defnition}
\newtheorem{property}{Property}

\title{Orthogonal and Idempotent Transformations for Learning Deep Neural Networks}

\author{
Jingdong Wang$^1$
\quad
Yajie Xing$^2$
\quad
Kexin Zhang$^2$
\quad
Cha Zhang$^1$
\\
$^1$ Microsoft Research
\quad
$^2$Peking University
}

\begin{document}
\maketitle
\begin{abstract}
Identity transformations,
used as skip-connections in residual networks,
directly
connect convolutional layers close
to the input
and those close to the output
in deep neural networks,
improving information flow
and thus easing the training.
In this paper,
we introduce two alternative linear transforms,
orthogonal transformation
and idempotent transformation.
According to the definition and property
of orthogonal and idempotent matrices,
the product of multiple orthogonal (same idempotent) matrices,
used to form linear transformations,
is equal to a single orthogonal (idempotent) matrix,
resulting in that information flow is improved
and the training is eased.
One interesting point is that the success essentially stems from
feature reuse
and gradient reuse in forward and backward propagation
for maintaining the information during flow
and eliminating the gradient vanishing problem
because of the express way through skip-connections.
We empirically demonstrate the effectiveness of the proposed
two transformations:
similar performance in single-branch networks
and even superior in multi-branch networks
in comparison to identity transformations.
\end{abstract}

\section{Introduction}
Training convolution neural networks
becomes more difficult
with the depth increasing
and even the training accuracy deceases
for very deep networks.
Identity mappings or transformations,
which are adopted as skip-connections
in deep residual networks~\cite{HeZRS16},
ease the training of very deep networks
and make the accuracy improved.

Identity transformations
lead to shorter connections
between layers close to the input and those close to the output.
It is shown that
identity transformations improve information flow
in both forward propagation and back-propagation
because
the product of identity matrices is still an identity matrix,
in other words,
\emph{multiple skip-connections is essentially like a single skip-connection}
no matter how many skip-connections there are.

In this paper,
we introduce two linear transformations
and use them as skip-connections
for improving information flow.
The first one is an orthogonal transformation.
Multiplying several orthogonal matrices,
used to form the orthogonal transformations,
yields
an orthogonal matrix.
The benefit is that
information attenuation and explosion is avoided
because the absolute values of the eigenvalues of an orthogonal matrix
are always $1$.
The second one is an idempotent transformation,
whose transformation matrix is an idempotent matrix
which, when multiplied by itself, yields itself.
A sequence of idempotent transformations
with the same idempotent matrices
is equivalent to a single idempotent transformation.
We show that the success essentially comes from
feature reuse
and gradient reuse in forward and backward propagation
for maintaining the information and eliminating the gradient vanishing problem
because of the express way through skip-connections.

The empirical results show that
single-branch deep neural networks with idempotent
and orthogonal transformations
as skip-connections achieve
perform similarly to
those with identity transformations
and that the performances are superior
when applied to multi-branch networks.

\section{Related Works}
In general, deeper convolutional neural networks leads to superior classification accuracy.
An example is the improvement
on the ImageNet classification
from AlexNet~\cite{KrizhevskySH12} ($7$ layers)
to VGGNet~\cite{SimonyanZ14a} ($19$ layers).
However,
going deeper increases the training difficulty.
Techniques to ease the training include optimization techniques~\cite{nair2010rectified, HeZRS15, clevert2015fast, IoffeS15, glorot2010understanding, mishkin2015all, neyshabur2015path}
and network architecture design.
In the following, we discuss representative works on network architecture design.

GoogLeNet~\cite{SzegedyLJSRAEVR15}
is one of the first works,
designing network architectures
to deal with the difficulty of training deep networks.
It is built by repeating Inception blocks
each of which contains short and long branches,
and thus there are both short and long paths
between layers close to the input layer
and those close to the output layer,
i.e., information flow is improved.

Inspired by Long Short-Term Memory recurrent networks, highway networks~\cite{srivastava2015training}
adopt identity transformations together with adaptive gating mechanism,
allowing computation paths along
which information can flow across many layers without attenuation.
It indeed eases the training of very deep networks, e.g., $100$ layers.
Residual networks~\cite{HeZRS16} also adopt identity transformations as skip-connections,
but without including gating units,
making training networks of thousands of layers easier.
In this paper, we introduce two alternative transformations,
orthogonal and idempotent transformations,
which also improve information flow.
We do not find that they learn residuals as claimed in~\cite{HeZRS16}
and but find that features and gradients are reused
through the express way composed of skip-connections.

FractalNets~\cite{LarssonMS16a}, deeply-fused nets~\cite{WangWZZ16}, and DenseNets~\cite{HuangLW16a}
present various multi-branch structures,
leading to short and long paths
between layers close to the input layer
and those close to the output layer.
Consequently, the effective depth~\cite{LarssonMS16a}
or the average depth~\cite{WangWZZ16} is reduced a lot
though the nominal depth is great
and accordingly information flow is improved.

Deep supervision~\cite{LeeXGZT15} associates a companion local output
and accordingly a loss function
with each hidden layer,
which results in shorter paths
from hidden layers
to the loss layers.
Its success provides an evidence
that effective depth is crucial.
FitNets~\cite{RomeroBKCGB14},
a student-teacher paradigm,
train a thinner and deeper student network
such that the intermediate representations
approach the intermediate representations of a wider and shallower (but still deep) teacher
network that is relatively easy to be trained,
which is in some sense a kind of deep supervision,
also reducing the effective depth.


\begin{figure}
\centering
\subfigure[]{\label{fig:blocks:original}\fbox{\includegraphics[height=0.25\linewidth]{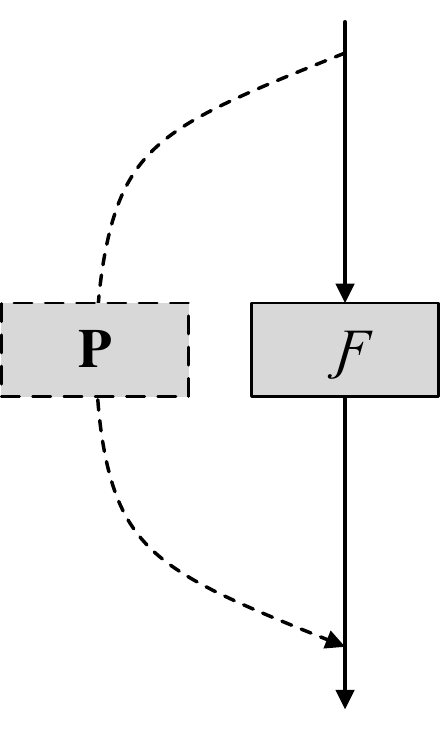}}}~~~~~~
\subfigure[]{\label{fig:blocks:orthogonal}\fbox{\includegraphics[height=0.25\linewidth]{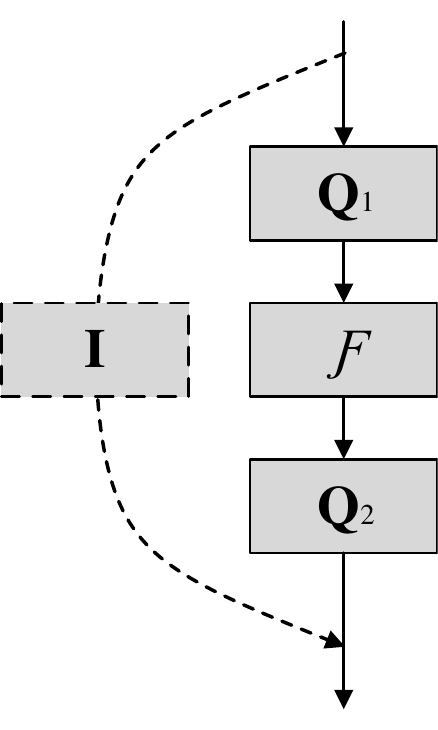}}}~~~~~~
\subfigure[]{\label{fig:blocks:idempotent}\fbox{\includegraphics[height=0.25\linewidth]{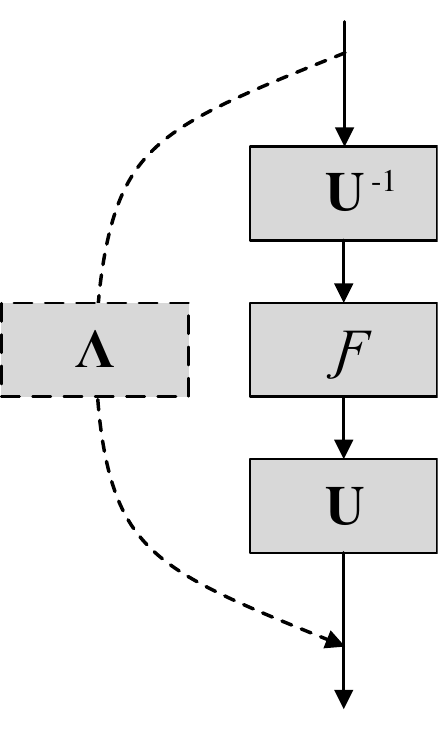}}}~~~~~~~~~
\subfigure[]{\label{fig:blocks:twobranch1}\fbox{\includegraphics[height=0.25\linewidth]{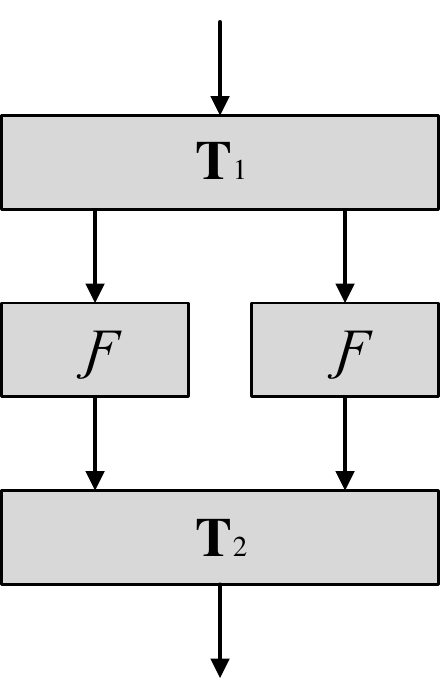}}}~~~~~~
\subfigure[]{\label{fig:blocks:twobranch2}\fbox{\includegraphics[height=0.25\linewidth]{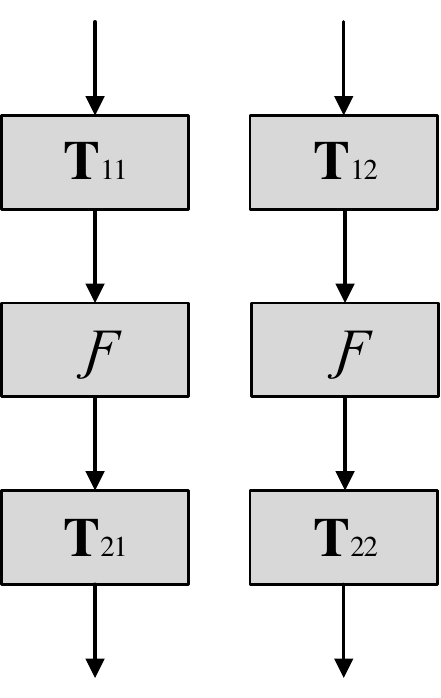}}}
\caption{(a) A building block with a linear transformation $\mathbf{P}$;
(b) An equivalent form for an orthogonal transformation;
$\mathbf{Q}_1$ and $\mathbf{Q}_2$
correspond to pre-orthogonal transformation
and post-orthogonal transformation.
(c) An equivalent form for an idempotent transformation.
(d) The regular connection with two branches converted
from an orthogonal or idempotent transformation,
where $\mathbf{T}_1$ and $\mathbf{T}_2$
are pre-linear transformation and post-linear transformation.
(e) Two separate branches,
where $\mathbf{T}_{11}$ ($\mathbf{T}_{12}$)
and $\mathbf{T}_{21}$ ($\mathbf{T}_{22}$)
are pre-linear transformation and post-linear transformation.
The network in (d) cannot be converted to that in (e)
for general pre- and post-transformations.
}
\label{fig:blocks}
\end{figure}

\section{Orthogonal and Idempotent Transformations}

A building block with a linear transformation
used as the skip-connection is written as:
\begin{align}
\mathbf{y} & = \mathbf{P}\mathbf{x} + \mathcal{F}(\mathbf{x}, \mathcal{W}).\label{eqn:buildingblock}
\end{align}
Here, $\mathbf{x}$ and $\mathbf{y}$ are the input and output features.
$\mathbf{P}\mathbf{x}$ is the skip-connection,
and $\mathbf{P}$ is the transformation matrix.
$\mathcal{F}$ is the regular connection,
e.g., two convolutional layers
with each followed by a ReLU activation function,
and $\mathcal{W}$ is the parameters of the function $\mathcal{F}$.
Following residual networks~\cite{HeZRS16},
we design the networks
by starting with a convolutional layer,
repeating such building blocks,
and appending a global pooling layer and a fully-connected layer.
Figure~\ref{fig:blocks:original} illustrates such a block.

The recursive equations below
show
how the features are forward-propagated
through building blocks
and the gradients are backward-propagated.

\noindent\textbf{Forward propagation.}
The transformation function,
transferring the feature $\mathbf{x}_m$, the input of the $m$th building block
to $\mathbf{x}_n$, the input of the $n$th building block,
is given as follows,
\begin{align}
\mathbf{x}_n = \mathbf{P}^{n-m}\mathbf{x}_m +
\sum\nolimits_{i=m}^{n-1}\mathbf{P}^{n-i-1} \mathbf{x}'_{i+1}, \label{eqn:forward}
\end{align}
where $\mathbf{x}'_{i+1} = \mathcal{F}(\mathbf{x}_i, \mathcal{W}_i)$,
and $\mathbf{x}_i$ and $\mathcal{W}_i$
are the input and the parameters of the $i$th block.

\noindent\textbf{Backward propagation.}
The gradient is backward-prorogated
from $\mathbf{x}_n$ to $\mathbf{x}_m$
as below,
\begin{align}
\frac{\partial \mathcal{L}}{\partial \mathbf{x}_m} = (\mathbf{P}^{n-m})^\top \frac{\partial \mathcal{L}}{\partial \mathbf{x}_n}   + \sum\nolimits_{i=m}^{n-1} \frac{\partial\mathbf{x}'_{i+1}}{\partial \mathbf{x}_m}
(\mathbf{P}^{n-i-1})^\top
\frac{\partial \mathcal{L}}{\partial \mathbf{x}_n} , \label{eqn:backward}
\end{align}
where $\mathcal{L}$ is the loss function.

In the following,
we show that
the feature $\mathbf{x}_m$ is \emph{reused} in any later feature $\mathbf{x}_n$
instead of only $\mathbf{x}_{m+1}$
and the gradient $\frac{\partial \mathcal{L}}{\partial \mathbf{x}_n}$
is \emph{reused} in any early gradient
$\frac{\partial \mathcal{L}}{\partial \mathbf{x}_m}$
such that \emph{the signal (information)
is maintained and
the vanishing problem is eliminated},
for identity transformations,
orthogonal transformations
and idempotent transformations.

\subsection{Identity transformations}
Identity transformations, i.e., $\mathbf{P} = \mathbf{I}$,
are adopted in residual networks~\cite{HeZRS16}.
The forward and backward processes are rewritten as below,
\begin{align}
\mathbf{x}_n & = \mathbf{x}_m + \sum\nolimits_{i=m}^{n-1} \mathbf{x}'_{i+1}, \label{eqn:forwardidentity} \\
\frac{\partial \mathcal{L}}{\partial \mathbf{x}_m} & = \frac{\partial \mathcal{L}}{\partial \mathbf{x}_n}  +  \sum\nolimits_{i=m}^{n-1}\frac{\partial\mathbf{x}'_{i+1}}{\partial \mathbf{x}_m} \frac{\partial \mathcal{L}}{\partial \mathbf{x}_n} . \label{eqn:backwardidentity}
\end{align}
It is obvious that
there is a path
along skip-connections,
where
(1) $\mathbf{x}_m$ directly flows to
$\mathbf{x}_n$ though there are $(n-m)$ blocks;
and (2)
the gradient with respect to $\mathbf{x}_n$
is directly backward sent to the gradient
with respect to $\mathbf{x}_m$
along the same path (both correspond to
the first term of the right-hand side
of the above two equations).

\subsection{Orthogonal transformations}
An orthogonal transformation is a linear transformation,
where the transformation matrix is orthogonal.
Mathematically,
a matrix $\mathbf{Q}$ is orthogonal
if $\mathbf{Q}^\top\mathbf{Q} = \mathbf{Q}\mathbf{Q}^\top = \mathbf{I}$.
We have the following property:
\begin{property}
The product of an arbitrary number of orthogonal matrices is orthogonal:
$\prod_{k=1}^K \mathbf{Q}_k$ is orthogonal,
if $\mathbf{Q}_1,~\mathbf{Q}_k,~\dots,~\mathbf{Q}_K$
are orthogonal.
\label{prop:orthogonalmatrixproduct}
\end{property}
Thus, the forward process (Equation~\ref{eqn:forward})
is rewritten as follows,
\begin{align}
\mathbf{x}_n = \mathbf{Q}_{n-m}\mathbf{x}_m +
\sum\nolimits_{i=m}^{n-1}\mathbf{Q}_{n-i-1} \mathbf{x}'_{i+1}, \label{eqn:forwardorthogonal}
\end{align}
where $\mathbf{Q}_{n-m} = \mathbf{Q}^{n-m}$
and $\mathbf{Q}_{n-i-1} = \mathbf{Q}^{n-i-1}$. We can see that
$\mathbf{x}_m$ is sent to $\mathbf{x}_n$
via a single orthogonal transformation (corresponding to the first term of the right-hand side)
One notable property is that
any orthogonal transformation $\mathbf{Q}$ preserves
the length of vectors:
$\|\mathbf{Q}\mathbf{x}\|_2 = \|\mathbf{x}\|_2$.
This means that
through the path formed by skip-connections,
\emph{the norm of the vector is maintained}.

The backward process (Equation~\ref{eqn:backward}) becomes
\begin{align}
\frac{\partial \mathcal{L}}{\partial \mathbf{x}_m} =  \mathbf{Q}_{n-m}^\top  \frac{\partial \mathcal{L}}{\partial \mathbf{x}_n}
+  \sum\nolimits_{i=m}^{n-1} \frac{\partial\mathbf{x}'_{i+1}}{\partial \mathbf{x}_m} \mathbf{Q}_{n-i-1}^\top \frac{\partial \mathcal{L}}{\partial \mathbf{x}_n}. \label{eqn:backwardorthogonal}
\end{align}
Again,
there is a path, formed by skip-connections
and behaving like \emph{a single orthogonal transformation layer},
where the gradient with respect to $\mathbf{x}_n$ is sent to
the gradient with respect to $\mathbf{x}_m$
no matter how many building blocks there are
between $\mathbf{x}_m$ and $\mathbf{x}_n$,
and \emph{the norm of the gradient is maintained}.

\noindent\textbf{Conversion to identity transformations.}
We show that
the orthogonal transformation can be \emph{absorbed}
into the regular connection
and the skip-connection is reduced to an identity transformation,
which is illustrated in Figure~\ref{fig:blocks:orthogonal}.

\begin{theorem}
For a network, which is formed
with orthogonal transformations as skip-connections,
there exists another network,which is formed
with identity transformations as skip-connections,
such that,
given an arbitrary input $\mathbf{x}_0$,
the final outputs of the two networks are the same.
\label{thm:equivalence}
\end{theorem}
\begin{proof}
We prove the theorem by considering the networks where the number of channels are not changed).
For the networks with the number of channels changed,
the proof is a little complex but similar.

The network with orthogonal transformations
($L$ building blocks) is mathematically formed
as below,
\begin{align}\label{xx}
\mathbf{x}^q_1 & = \mathbf{W}^q_0 \mathbf{x}_0, \\
\mathbf{x}^q_{i+1} & = \mathbf{Q} \mathbf{x}^q_i +
\mathcal{F}^q(\mathbf{x}^q_i, \mathcal{W}^q_i), i=1,2,\dots,L.
\end{align}

We construct a network
with identity transformations,
\begin{align}\label{xx}
\mathbf{x}_1 & = \mathbf{W}_0 \mathbf{x}_0, \\
\mathbf{x}_{i+1} & = \mathbf{x}_i + \mathcal{F}(\mathbf{x}_i, \mathcal{W}_i),
i=1,2,\dots,L,
\end{align}
satisfying
\begin{align}
\mathbf{W}_0 & =   \mathbf{Q}^L  \mathbf{W}^q_0, \\
\mathcal{F}(\mathbf{x}_i, \mathcal{W}_i) & = \mathbf{Q}^{L-i} \mathcal{F}^q(\mathbf{x}^q_i, \mathcal{W}^q_i) \mathbf{Q}^{i-L-1}.
\end{align}

It can be easily verified that
given an arbitrary input $\mathbf{x}_0$,
the final outputs of the two networks are the same:
$\mathbf{x}^p_{L+1} = \mathbf{x}_{L}$.

Thus, the theorem holds.
\end{proof}

\subsection{Idempotent transformations}
\label{sec:idempotenttransformations}
An idempotent transformation is a linear transformation
in which the transformation matrix is an idempotent matrix.
An idempotent matrix is a matrix which, when multiplied by itself, yields itself.
\begin{definition}[Idempotent matrix]
The matrix $\mathbf{P}$
is idempotent if and only if $\mathbf{P}\mathbf{P} = \mathbf{P}$,
or equivalently $\mathbf{P}^k = \mathbf{P}$,
where $k$ is a positive integer.
\end{definition}

The forward process (Equation~\ref{eqn:forward})
becomes
\begin{equation}
\mathbf{x}_n = \mathbf{P}\mathbf{x}_m + \sum\nolimits_{i=m}^{n-1}\mathbf{P}^{n-i-1} \mathbf{x}'_{i+1} \label{eqn:forwardIdempotent},
\end{equation}
where $\mathbf{P}^{n-i-1} = \mathbf{I}$ when $n=i+1$,
and $\mathbf{P}$
otherwise.
This implies that
$\mathbf{x}_m$ is directly sent to
$\mathbf{x}_n$
through skip-connections that behave like a single skip-connection.

Similarly,
the backward process (Equation~\ref{eqn:backward}) becomes
\begin{equation}
\frac{\partial \mathcal{L}}{\partial \mathbf{x}_m} =
\mathbf{P}^\top \frac{\partial \mathcal{L}}{\partial \mathbf{x}_n}
+ \sum\nolimits_{i=m}^{n-1}\frac{\partial \mathbf{x}'_{i+1} }{\partial \mathbf{x}_l}  (\mathbf{P}^{n-i-1})^\top \frac{\partial \mathcal{L}}{\partial \mathbf{x}_n}  \label{eqn:backwardIdempotent},
\end{equation}
which again implies that
the gradient with respect to $\mathbf{x}_n$ is directly sent to
the gradient with respect to $\mathbf{x}_m$
through the skip-connections.

\noindent\textbf{Information maintenance.}
Different from
identity transformations and orthogonal transformations,
idempotent transformations maintain \emph{the vector
lying in the column space of~$\mathbf{P}$}:
\begin{align}
\mathbf{P}\mathbf{v} = \mathbf{v},
\end{align}
where $\mathbf{v} = \mathbf{P}\mathbf{x}$
and $\mathbf{x}$ is an arbitrary $d$-dimensional vector.

Apparently,
when a vector $\bar{\mathbf{v}}$ lies in the null space of the column space,
i.e., $\mathbf{P} \bar{\mathbf{v}} = 0$,
it looks that
the skip connections do not help to improve information flow.
Considering another term
in the right-hand side of Equation~\ref{eqn:forwardIdempotent},
\begin{align}
\mathbf{P}^{n-m-1}\mathbf{x}'_{m+1}
= \mathbf{P}^{n-m-1}\mathcal{F}(\mathbf{x}_m, \mathcal{W}_m)
= \mathbf{v}',
\end{align}
we can find $\mathcal{W}$,
if $\mathcal{F}$ is formed by convolutional layers
and ReLU layers,
such that
\begin{align}
\mathbf{P}\mathbf{v}' = \mathbf{v}' \neq \mathbf{0},
\label{eqn:secondpathidempotent}
\end{align}
which means that there is a path
along which the information
does not vanish.
Vanishing is still possible
though rare,
e.g., in the case
$\mathbf{P}\mathbf{x} + \mathcal{F}(\mathbf{x}, \mathcal{W})$
lies in the null space of $\mathbf{P}$.
There is similar analysis for gradient maintenance.

\noindent\textbf{Diagonalization.}
It is known that an idempotent matrix is diagonalizable:
$\mathbf{P} = \mathbf{U}^{-1}\boldsymbol{\Lambda} \mathbf{U}$,
where $\boldsymbol{\Lambda}$ is a diagonal matrix
whose diagonal entries are $0$ or $1$
and $\mathbf{U}$ is invertible.
We illustrate it in Figure~\ref{fig:blocks:idempotent}.
A network containing $L$ blocks formed with idempotent transformations
written as follows,
\begin{align}\label{xx}
\mathbf{x}^p_1 & = \mathbf{W}^p_0 \mathbf{x}_0, \\
\mathbf{x}^p_{i+1} & = \mathbf{P} \mathbf{x}^p_i +
\mathcal{F}^p(\mathbf{x}^p_i, \mathcal{W}^p_i), i=1,2,\dots,L,
\end{align}
can be transferred to a network
with skip connections formed by linear transformations
whose transformation matrix is a diagonal matrix $\boldsymbol{\Lambda}$ composed of $0$ and $1$:
\begin{align}\label{xx}
\mathbf{x}_1 & = \mathbf{U}  \mathbf{W}^p_0 \mathbf{x}_0, \\
\mathbf{x}_{i+1} & = \boldsymbol{\Lambda} \mathbf{x}_i + \mathbf{U} \mathcal{F}(\mathbf{x}_i, \mathcal{W}_i) \mathbf{U}^{-1},
i=1,2,\dots,L-1, \\
\mathbf{x}_{L+1} & = \mathbf{U}^{-1}(\boldsymbol{\Lambda} \mathbf{x}_L +  \mathbf{U}\mathcal{F}(\mathbf{x}_L, \mathcal{W}_L) \mathbf{U}^{-1}).
\end{align}

\subsection{Discussions}
We can easily show that
\emph{identity transformations are idempotent and orthogonal transformations
as $\mathbf{I}^k = \mathbf{I}$ (idempotent)
and $\mathbf{I}^\top \mathbf{I} = \mathbf{I}$ (orthogonal)}.
Here we discuss a bit more on feature reuse,
gradient back-propagation,
and extensions.

\noindent\textbf{Feature reuse.}
We have generalized identity transformations
to orthogonal transformations
and idempotent transformations,
to eliminate information vanishing and explosion.
One point we want to make clearer
is that Equations~\ref{eqn:forwardidentity},~\ref{eqn:forwardorthogonal}
and~\ref{eqn:forwardIdempotent}
(for forward propagation)
hold for any $\mathbf{x}_{m}$,
$m=1, 2, \dots, n-1$.
In other words,
$\mathbf{x}_n$ reuses
all the previous features:
$\mathbf{x}_{1}, \mathbf{x}_{2}, \dots, \mathbf{x}_{n-1}$
rather than only $\mathbf{x}_{n-1}$.
We have similar observation on gradient reuse.

\noindent\textbf{Gradient back-propagation.}
Considering the network without skip-connections,
back-propagating the gradient $\mathbf{g}$
through $L$ regular connections
yields the gradient $\mathbf{g}_1$ with respect to
$\mathbf{x}_1$:
$\mathbf{g}_1 = (\prod_{i=1}^L \frac{\partial \mathcal{F}(\mathbf{x}_i, \mathcal{W})}{\partial \mathbf{x}_i}) \mathbf{g}$.
With linear transformations as skip-connections,
$\mathbf{g}_1 = \prod_{i=1}^L (\mathbf{P}^\top +
\frac{\partial \mathcal{F}(\mathbf{x}_i, \mathcal{W})}{\partial \mathbf{x}_i}) \mathbf{g}$.
One reason for gradient vanishing
($\mathbf{g}_1 \approx \mathbf{0}$)
means that $\mathbf{g}$ lies in the null space of
$\prod_{i=1}^L \frac{\partial \mathcal{F}(\mathbf{x}_i, \mathcal{W})}{\partial \mathbf{x}_i} \approx \mathbf{0}$.
Adding a proper $\mathbf{P}^\top$ to each $\frac{\partial \mathcal{F}(\mathbf{x}_i, \mathcal{W})}{\partial \mathbf{x}_i}$
in some sense shrinks the null space,
and reduces the chance of gradient vanishing.
It is as expected that a transformation with higher-rank $\mathbf{P}$
leads to lower chance of gradient vanishing.
We empirically justify it in Figure~\ref{fig:rankaccuracy}.

\noindent\textbf{Extension of idempotent transformations.}
We extend idempotent transformations:
$\mathbf{P} = \mathbf{U}^{-1} \boldsymbol{\Lambda}\mathbf{U}$,
by relaxing the diagonal entries (eigenvalues) in $\boldsymbol{\Lambda}$.
The relaxed conditions are
(i) the absolute values of diagonal entries are not larger than $1$
and (ii) there is as least one diagonal entry whose absolute value is $1$.
Considering a special case that the absolute values of eigenvalues are only $0$ or $1$,
the absolute vector in the column space of $\mathbf{P}$ is maintained
is: $|\mathbf{P}^k\mathbf{v}| = |\mathbf{v}|$.
A typical example is a periodic matrix:
$\mathbf{P}^{N+1} = \mathbf{P}$,
where $N$ is a positive integer.

\subsection{Multi-branch networks}
The multi-branch networks have been studied
in many recent works~\cite{AbdiN16, TargAL16, XieGDTH16, ZhaoWLTZ16}.
We study the application of
the two linear transformations
to multi-branch structures.

\noindent\textbf{Orthogonal transformations.}
Theorem~\ref{thm:equivalence} shows
how orthogonal transformations
are transformed to identity transformations.
and still holds
for multi-branch structures.
Figure~\ref{fig:blocks:twobranch1} depicts
the regular connection in the block converted
from a block with an orthogonal transformation
for multiple branches.
We can see that the converted regular connection
cannot be separated into multiple branches
(as shown in Figure~\ref{fig:blocks:twobranch2})
because of two extra transformations:
pre-transformation $\mathbf{Q}^{i-L-1}$
(shortened as $\mathbf{T}_1$ in Figure~\ref{fig:blocks:twobranch1})
and post-transformation $\mathbf{Q}^{L-i}$
(shortened as $\mathbf{T}_2$ in Figure~\ref{fig:blocks:twobranch1}).
The two transformations are essentially
$1 \times 1$ convolutions,
which exchange the information
across the branches.
Without the two transformations
(e.g., in residual networks using identity transformation),
there is no interaction
across these branches.

\noindent\textbf{Idempotent transformations.}
We have shown that idempotent transformations
can be transformed to diagonalized idempotent transformations.
There are two extra transformations in regular connections:
pre-transformation $\mathbf{U}^{-1}$
and post-transformation $\mathbf{U}$
(see Figure~\ref{fig:blocks:idempotent}).
But the diagonal entries of the diagonal idempotent matrix
are $0$ and $1$.
Compared to identity transformations,
the regular connection contains two extra $1\times 1$ convolutions,
which is similar to orthogonal transformations
and results in information exchange across the branches.

\section{Experiments}
\subsection{Datasets}
\noindent\textbf{CIFAR.}
CIFAR-$10$ and CIFAR-$100$~\cite{Alex2009} are
subsets of the $80$ million tiny image database~\cite{TorralbaFF08}.
Both datasets contain $60,000$ $32 \times 32$ color images
with $50,000$ training images and $10,000$ testing images.
The CIFAR-$10$ dataset includes $10$ classes,
each containing $6,000$ images,
$5,000$ for training and $1,000$ for testing.
The CIFAR-$100$ dataset includes $100$ classes,
each containing $600$ images,
$500$ for training and $100$ for testing.
We follow a standard data augmentation scheme widely used for this dataset~\cite{HeZRS16, HuangLW16a, HuangSLSW16, LeeXGZT15, LinCY13, XieGDTH16}: We first zero-pad the images with $4$ pixels
on each side, and then randomly crop them to produce $32\times32$ images,
followed by horizontally mirroring half of the images.
We normalize the images by using the channel means and standard deviations.

\noindent\textbf{SVHN.}
The Street View House Numbers (SVHN) dataset is obtained from house numbers in Google Street View images.
It contains $73,257$ training images, $26,032$ testing images and $531,131$ additional training images. Following~\cite{HuangSLSW16, LeeXGZT15, LinCY13}, we select out $400$ samples per class from the training set
and $200$ samples from the additional set,
and use the remaining images
as the training set without any data augmentation.

\subsection{Setup}
\noindent\textbf{Networks.}
The network starts with
a $3\times3$ convolutional layer,
$3$ stages,
where each stage contains $K$ building blocks
and
there are two downsampling layers,
and a global pooling layer followed by a fully-connected layer
outputting the classification result.

In our experiments,
we consider $3$ kinds of regular connections forming building blocks:
(a) single branch,
(b) $4$ branches,
and (c) depthwise convolution
(an extreme multi-branch connection,
each branch contains one channel).
Each branch consists of batch normalization,
convolution, batch normalization,
ReLU and convolution (BN ),
which is similar to the pre-activation residual connection~\cite{HeZRS16ECCV}.
We empirically study two idempotent transformations
and two orthogonal transformations
and compare them with identity transformations.

\noindent\emph{Idempotent transformations:}
The first one is a merge-and-run style~\cite{ZhaoWLTZ16, ZhaoWLTZ17},
denoted by Idempotent-MR
\begin{align}
\mathbf{P}_{\texttt{MR}}=\frac{1}{B}\begin{bmatrix}
\mathbf{I} & \cdots & \mathbf{I} \\
\vdots & \; & \vdots \\
\mathbf{I} & \cdots & \mathbf{I} \\
\end{bmatrix},
\end{align}
which
is a block matrix containing $B \times B$ blocks
with $B$ being the number of branches,
and each block is an identity matrix.
The second one is obtained
by subtracting $\mathbf{P}_{\texttt{MR}}$ from the identity matrix:
\begin{align}
\mathbf{P}_{\overline{\texttt{MR}}} = \mathbf{I} - \mathbf{P}_{{\texttt{MR}}},
\end{align}
which is named as Idempotent-CMR (c=complement).
The ranks of the two matrices
are $\frac{R}{B}$ and $R - \frac{R}{B}$,
where $R$ is the size of the matrix or
the total number of channels.

\emph{Orthogonal transformations:}
The first one is built from Kronecker product:
$\mathbf{P}=\mathbf{M} \otimes \mathbf{M} \cdots \otimes \mathbf{M}$,
where $\otimes$ is the Kronecker product operation,
and
\begin{align}
\mathbf{M} = \frac{1}{\sqrt{2}}
   \begin{bmatrix}
	 1 & -1 \\
	 1 & 1
   \end{bmatrix}.
\end{align}
We name it Orthogonal-TP.
The second one is
a random orthogonal transformation,
named Orthogonal-Random,
also constructed using Kronecker product.
In each block, we generate different orthogonal transformations.

\begin{wrapfigure}{r}{0pt}
\includegraphics[width = 0.4\textwidth]{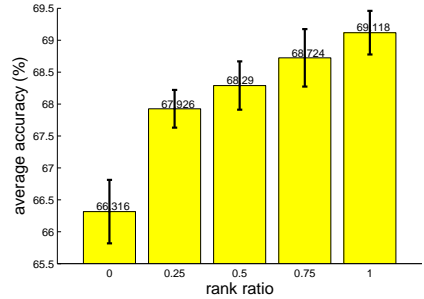}
\caption{Illustrating how the CIFAR-$100$ classification
accuracy changes
when changing the rank of the idempotent matrix
over a $20$-layer network.
(changing $B$ when designing $\mathbf{P}_{\texttt{MR}}$).
$0$ corresponds to that there is no skip-connection.}
\label{fig:rankaccuracy}
\end{wrapfigure}

\noindent\textbf{Training.}
We use the SGD algorithm with the Nesterov momentum
to train all the networks for $300$ epochs
on CIFAR-$10$/CIFAR-$100$ and $40$ epochs
on SVHN,
both with a total mini-batch size $64$.
The initial learning rate is set to $0.1$,
and is divided by $10$ at $1/2$ and $3/4$ of the total number of training epochs.
Following residual networks~\cite{HeZRS15},
the weight decay is $0.0001$,
the momentum is $0.9$,
and the weights are initialized as in residual networks~\cite{HeZRS15}.
Our implementation is based on Keras and TensorFlow~\cite{AbadiABBCCCDDDG16}.

\subsection{Results}
\noindent\textbf{Single-branch.}
We compare four skip-connections:
Identity transformation, Idempotent-CMR, Orthogonal-TP and Orthogonal-Random.
To form the idempotent matrix for Idempotent-CMR,
we set $B$ to be the number of the channels,
i.e., $\mathbf{P}_{{\texttt{MR}}}$ is a matrix
with all entries being $\frac{1}{B}$.
We do not evaluate Idempotent-MR because in this case
its rank is only $1$,
whose performance is expected to be low.
In general, idempotent transformations
with lower ranks
perform worse than
those with higher ranks.
This is empirically verified in Figure~\ref{fig:rankaccuracy}.

Table~\ref{table:single-branch}
shows the results
over networks of
depth $20$ and $56$,
containing $9$ and $27$ building blocks,
respectively.
One can see that the results with idempotent transformations
and orthogonal transformations are similar
to those with identity transformations:
empirically demonstrating
that idempotent transformations
and orthogonal transformations
improve information flow.

\begin{table}[t]
\caption{Comparing classification accuracies
of identity transformations,
idempotent transformations,
and orthogonal transformations
for single-branch regular connections.
$a+b$: $a$ is the average accuracy over five runs
and $b$ is the standard deviation.}
\footnotesize
		  \label{table:single-branch}
		  \centering
		  \begin{tabular}{llllll}
\hline
				\multirow{3}{*}{Depth}&\multirow{3}{*}{Mapping}&\multirow{3}{*}{Width}
				&\multicolumn{3}{c}{Accuracy}\\
				\cline{4-6}
				&&& CIFAR-$10$& CIFAR-$100$&SVHN\\
\hline
				\multirow{4}{*}{20}
				&Identity&16,32,64 & $92.468\pm0.171$ & $\mathbf{69.118\pm0.341}$ & $97.574\pm0.109$\\
				&Idempotent-CMR&16,32,64 & $92.580\pm0.088$ & $69.048\pm0.545$ & $97.570\pm0.056$\\
				&Orthogonal-TP&16,32,64 & $\mathbf{92.628\pm0.158}$ & $69.016\pm0.509$ & $\mathbf{97.616\pm0.068}$\\
				&Orthogonal-Random&16,32,64 & $92.598\pm0.097$ & $68.716\pm0.305$ & $97.548\pm0.087$\\
\hline
				\multirow{4}{*}{56}
				&Identity&16,32,64 & $94.396\pm0.200$ & $\mathbf{73.410\pm0.350}$ & $97.794\pm0.083$\\
				&Idempotent-CMR&16,32,64 & $\mathbf{94.452\pm0.267}$ & $73.058\pm0.210$ & $\mathbf{97.794\pm0.075}$\\
				&Orthogonal-TP&16,32,64 & $94.316\pm0.299$ & $73.352\pm0.312$ & $97.778\pm0.081$\\
				&Orthogonal-Random&16,32,64 & $94.400\pm0.206$ & $73.288\pm0.209$ & $97.774\pm0.059$\\
\hline
			\end{tabular}
		\end{table}

\begin{table}[t]
\caption{Comparing classification accuracies
for multiple-branch regular connections.
The width of each branch is the same
and the total width is described in the table.}
\footnotesize
		  \label{table:4-branch}
		  \centering
		  \begin{tabular}{llllll}
\hline
				\multirow{3}{*}{Depth}&\multirow{3}{*}{Mapping}&\multirow{3}{*}{Width}
				&\multicolumn{3}{c}{Accuracy}\\
				\cline{4-6}
				&&& CIFAR-$10$ & CIFAR-$100$ &SVHN\\
\hline
				\multirow{4}{*}{20}
				&Identity&32,64,128 & $91.416\pm0.308$ & $67.002\pm0.153$ & $97.454\pm0.095$\\
				&Idempotent-CMR&32,64,128 & $92.186\pm0.238$ & $68.520\pm0.328$ & $97.542\pm0.040$\\
				&Idempotent-MR&32,64,128 & $92.454\pm0.101$ & $\mathbf{69.067\pm0.242}$ & $\mathbf{97.594\pm0.051}$\\
				&Orthogonal-TP&32,64,128 & $\mathbf{92.538\pm0.256}$ & $69.044\pm0.439$ & $97.546\pm0.058$\\
				&Orthogonal-Random&32,64,128 & $92.526\pm0.147$ & $68.668\pm0.432$ & $97.568\pm0.025$\\
				\midrule
				\multirow{4}{*}{32}
				&Identity&32,64,128 & $-$ & $68.158\pm0.119$ & $-$\\
				&Idempotent-CMR&32,64,128 & $-$ & $70.372\pm0.419$ & $-$\\
				&Idempotent-MR&32,64,128 & $-$ & $\mathbf{70.942\pm0.331}$ & $-$\\
				&Orthogonal-TP&32,64,128 & $-$ & $70.786\pm0.370$ & $-$\\
				&Orthogonal-Random&32,64,128 & $-$ & $70.926\pm0.307$ & $-$\\
\hline
				\multirow{4}{*}{56}
				&Identity&32,64,128 & $92.502\pm0.224$ & $69.524\pm0.346$ & $97.616\pm0.102$\\
				&Idempotent-CMR&32,64,128 & $93.724\mp0.135$ & $72.192\pm0.304$ & $97.766\pm0.102$\\
				&Idempotent-MR&32,64,128 & $93.858\pm0.136$ & $72.444\pm0.143$ & $\mathbf{97.816\pm0.069}$\\
				&Orthogonal-TP&32,64,128 & $93.780\pm0.289$ & $\mathbf{72.716\pm0.314}$ & $97.724\pm0.079$\\
				&Orthogonal-Random&32,64,128 & $\mathbf{93.886\pm0.125}$ & $72.476\pm0.264$ & $97.784\pm0.052$\\
\hline
				\multirow{4}{*}{110}
				&Identity&32,64,128 & $-$ & $71.112 \pm 0.375$ & $-$\\
				&Idempotent-CMR&32,64,128 & $-$ & $73.822 \pm 0.483$ & $-$\\
				&Idempotent-MR&32,64,128 & $-$ & $73.584\pm0.039$ & $-$\\
				&Orthogonal-TP&32,64,128 & $-$ & $74.086\pm0.333$ & $-$\\
				&Orthogonal-Random&32,64,128 & $-$ & $74.010 \pm 0.299$ & $-$\\
\hline
			\end{tabular}
		\end{table}

\noindent\textbf{Four-branch.}
We compare the results over
the networks,
where each regular connection consists of four branches.
The results are shown in Table \ref{table:4-branch}.
One can see that the idempotent and orthogonal transformations
perform better than identity transformations.
The reason is that compared to identity transformations,
the designed idempotent
and orthogonal transformations
introduce interactions across the four branches.

\noindent\textbf{Depth-wise.}
We evaluate the performance
over extreme multi-branch networks:
depth-wise networks,
where each branch only contains a single channel.
Table \ref{table:depthwise}
shows the results.
One can see that the comparison
is consistent to the $4$-branch case.
		\begin{table}[t]
		  \caption{Comparing classification accuracies
for depth-wise networks.}
		  \label{table:depthwise}
\footnotesize
		  \centering
		  \begin{tabular}{llllll}
\hline
				\multirow{3}{*}{Depth}&\multirow{3}{*}{Mapping}&\multirow{3}{*}{Width}
				&\multicolumn{3}{c}{Accuracy}\\
				\cline{4-6}
				&&&CIFAR-$10$ &CIFAR-$100$ &SVHN\\
			\hline
				\multirow{4}{*}{20}
				&Identity&64, 128,256 & $83.306\pm0.243$ & $58.388\pm0.303$ & $93.938\pm0.138$\\
				&Idempotent-CMR&64,128,256 & $84.990\pm0.389$ & $\mathbf{60.360\pm0.260}$ & $95.488\pm0.261$\\
				&Orthogonal-TP&64,128,256 & $\mathbf{85.456\pm0.270}$ & $58.828\pm0.225$ & $\mathbf{95.870\pm0.065}$\\
				&Orthogonal-Random&64,128,256 & $84.800\pm0.555$ & $58.262\pm0.415$ & $95.816\pm0.079$\\
		\hline
				\multirow{4}{*}{56}
				&Identity&64,128,256 & $85.048\pm0.354$ & $60.242\pm0.443$ & $94.938\pm0.099$\\
				&Idempotent-CMR&64,128,256 & $87.658\pm0.217$ & $\mathbf{63.068\pm0.396}$ & $96.818\pm0.026$\\
				&Orthogonal-TP&64,128,256 & $87.932\pm0.246$ & $62.836\pm0.465$ & $\mathbf{97.134\pm0.027}$\\
				&Orthogonal-Random&64,128,256 & $\mathbf{88.218\pm0.207}$ & $62.152\pm0.312$ & $97.058\pm0.055$\\
\hline
			\end{tabular}
		\end{table}

\section{Conclusions}
We introduce two linear transformations,
orthogonal and idempotent transformations,
which, we show theoretically and empirically,
behave like identity transformations,
improving information flow
and easing the training.
One interesting point is that the success stems from
feature and gradient reuse
through the express way composed of skip-connections,
for maintaining the information during flow
and eliminating the gradient vanishing problem.

\bibliographystyle{plain}
\bibliography{References}

\end{document}